\documentclass[11pt]{article}
\usepackage{booktabs} 
\usepackage{format}
\usepackage{macros}
\usepackage{hyperref}
\usepackage{url}
\usepackage{blindtext}
\usepackage[utf8]{inputenc}
\usepackage[T1]{fontenc}
\usepackage[draft,inline,nomargin,index]{fixme}
\usepackage{array,amssymb,amsmath,graphicx,comment,multirow,microtype}
\usepackage{booktabs,color,xspace,balance,csquotes}

\begin{document}
\title{Preference-Informed Fairness}

\author{
Michael P. Kim\thanks{Stanford University.   Part of this work completed while visiting the Weizmann Institute of Science. Supported, in part, by a Google Faculty Research Award, CISPA Center for Information Security, and the Stanford Data Science Initiative.
\texttt{mpk@cs.stanford.edu}} 
\and
Aleksandra Korolova\thanks{University of Southern California.  Part of this work completed while visiting the Weizmann Institute of Science.  Another part completed while visiting the Simons Institute for the Theory of Computing. \texttt{korolova@usc.edu}} 
\and
Guy N. Rothblum\thanks{Weizmann Institute of Science.  \texttt{rothblum@alum.mit.edu}. Research supported by the ISRAEL SCIENCE FOUNDATION (grant number 5219/17).} 
\and
Gal Yona\thanks{Weizmann Institute of Science.  \texttt{gal.yona@weizmann.ac.il}. Research supported by the ISRAEL SCIENCE FOUNDATION (grant number 5219/17). 
} 
}
\setcounter{footnote}{0}

\maketitle

\begin{abstract}
In this work, we study notions of fairness in decision-making systems when individuals have diverse preferences over the possible outcomes of the decisions.
Our starting point is the seminal work of Dwork \emph{et al.} [ITCS 2012] which introduced a notion of \emph{individual fairness} (IF): given a task-specific similarity metric, every pair of individuals who are similarly qualified according to the metric should receive similar outcomes. 
We show that when individuals have diverse preferences over outcomes, requiring IF may unintentionally lead to less-preferred outcomes for the very individuals that IF aims to protect (e.g.\ a protected minority group).
A natural alternative to IF is the classic notion of fair division, \emph{envy-freeness} (EF): no individual should prefer another individual's outcome over their own.
Although EF allows for solutions where all individuals receive a highly-preferred outcome, EF may also be overly-restrictive for the decision-maker.
For instance, if many individuals agree on the best outcome, then if any individual receives this outcome, they all must receive it,  regardless of each individual's underlying qualifications for the outcome.

We introduce and study a new notion of \emph{preference-informed individual fairness} (PIIF) that is a relaxation of both individual fairness and envy-freeness.
At a high-level, PIIF requires that outcomes satisfy IF-style constraints, but allows for deviations provided they are in line with individuals' preferences.
We show that PIIF can permit outcomes that are more favorable to individuals than any IF solution, while providing considerably more flexibility to the decision-maker than EF.
In addition, we show how to efficiently optimize any convex objective over the outcomes subject to PIIF for a rich class of individual preferences.
Finally, we demonstrate the broad applicability of the PIIF framework by extending our definitions and algorithms to the multiple-task targeted advertising setting introduced by Dwork and Ilvento [ITCS 2019].
\end{abstract}

\newcommand{\remove}[1]{}

\pagenumbering{gobble}

\newpage
\pagenumbering{arabic}

\section{Introduction}
\label{sec:intro}
Increasingly, algorithms are used to make consequential decisions about individuals.
Examples range from determining which content users see online to deciding which applicants are considered in lending and hiring decisions.
Automated decision-making comes with benefits, but it also raises substantial societal concerns (cf. \cite{o2017weapons} for a recent perspective). One prominent concern
is that these algorithms might discriminate against individuals or groups in a way that violates laws or social and ethical norms~\cite{ali2019discrimination, tobin2019pb, amazon2018hiring, AlltheWa25:online}. Thus there is an urgent need for frameworks and tools to mitigate the risks of algorithmic discrimination. A growing literature 
attempts to tackle these challenges by exploring different fairness criteria and ways to achieve them.

One prominent framework for establishing fairness in algorithmic decision-making systems comes from the seminal work of Dwork~\emph{et al.} \cite{dwork2012fairness}, which introduced the notion of {\em individual fairness} (IF).
IF relies on a task-specific similarity metric that specifies, for every pair of individuals, how similar they are with respect to the task at hand.
Given such a metric, individual fairness requires that similar individuals (according to the metric) be treated similarly, i.e.,\ assigned similar outcome distributions.
This is formalized via a Lipschitz condition, requiring that for any two individuals $i$ and $j$, the distance between their outcome distributions is bounded by their distance according to the metric.
Although coming up with a good metric can be challenging, metrics arise naturally in prominent existing examples (e.g.\ credit or insurance risk scores), and in natural scenarios (e.g.\ a metric specified by an external regulator).
Given an appropriate metric, individual fairness provides powerful protections from discrimination.

\paragraph{Accounting for individuals' preferences.} Our work is motivated by settings in which individuals may hold
diverse preferences over the possible outcomes.
Natural examples of such settings include recommendation systems on professional employment websites where job-searchers have diverse considerations (geography, work-life balance, company culture, etc.) that affect their interest in potential employers, and targeted advertising systems where different users have a wide variety of preferences over the subset of ads they'd like to see out of an enormous set of possibilities.
While the metric-based IF constraints prevent myriad forms of discrimination that can arise in automated decision-making systems, 
we argue that when individuals have different preferences over outcomes, IF can be too restrictive.
Specifically, we show that in such settings, ignoring individuals' preferences (as IF does) can come at a high cost to \emph{the very individuals that IF aims to protect}. 

We illustrate this observation using a simple example.
Consider a university organizing a career expo focused on software developer positions.
The university would like to assign each graduating student to (at most) a single interview slot with a prospective employer. 
To prevent discrimination, the university would like to enforce individual fairness.
For simplicity, we assume that there is an unbiased metric for judging qualifications for software development roles across employers based on GPA in the CS major.
Consider candidates $i$, $j$, and $k$, who are all similarly qualified, and suppose there are three employers $X$, $Y$, and $Z$.  When the candidates are polled for their preferences, $i$ prefers $X \succ Y \succ Z$, $j$ prefers $Y \succ Z \succ X$, and $k$ prefers $Z \succ X \succ Y$ (possibly due to geographic and work-life balance considerations). Despite the diversity of their preferences, since $i$, $j$, and $k$ are all similarly qualified, IF requires that the candidates receive similar distributions over interviews with the employers $X$, $Y$, and $Z$.
Thus, IF {\em rules out the allocation where each candidate gets their most-preferred interview}.

This toy example demonstrates that IF can be overly-restrictive, preventing some solutions where every individual is very happy with their outcome.
Moreover, under IF, even the most socially-conscious decision-maker may be forced to disregard the preferences of some groups of individuals in order to satisfy the constraints.
For example, if a decision-maker is required by IF to give similar members of majority and minority populations similar outcomes, then the 
decision-maker may choose the IF solution that gives everyone the outcome preferred by the majority, running the risk of ignoring the
preferences of historically-marginalized groups of individuals.

Faced with this shortcoming of IF, we consider alternative notions of fairness that may be better suited to handle settings where individuals hold rich preferences over outcomes.
The most natural alternative notion is \emph{envy-freeness} (EF) \cite{varian1974efficiency, foley1967resource}, a classic game-theoretic concept of fair division.
A set of outcomes is said to be envy-free if no individual prefers the outcome given to any other individual over their own.
At first glance, EF seems like a promising solution concept that addresses the concerns raised about IF: the decision where every individual receives their most-preferred outcome is EF. Indeed, Balcan \emph{et al.} \cite{balcan2018envy} recently presented EF as an alternative to IF in the context of fair classification.

However, we argue that EF may also be overly-restrictive, constraining the decision-maker in unreasonable ways.
Returning to the example of the career expo, suppose another individual $\ell$ has similar preferences to $i$, ($X \succ Y \succ Z$), but $\ell$ has a significantly lower GPA than $i$.
Consider an allocation where $i$ receives their most-preferred interview $X$, but $\ell$ does not receive any interview.
In this case, $\ell$ envies $i$ so this solution does not satisfy EF; nevertheless, the solution is reasonable from a fairness perspective.
Since $i$ has a much better GPA than $\ell$, it doesn't seem unfair to give $i$ the interview with $X$ over $\ell$, especially if the interview spots are limited.

This expanded example highlights the need for distinguishing between outcome distributions that might make some (or even all) individuals \emph{unhappy}, from distributions that are \emph{unfairly discriminatory};
articulating this distinction was an important conceptual contribution of the definition of individual fairness \cite{dwork2012fairness}.
Indeed, the unqualified individual $\ell$ might be unhappy that they do not receive an interview;
further, they might be even less happy when they see that the qualified individual $i$ received an interview with their top choice $X$.
In the eyes of the task-specific similarity metric, however, these two individuals are \emph{different} -- according to their GPAs, one is qualified, the other -- unqualified.
Thus, IF does not consider such an outcome discriminatory.
Furthermore, deciding to assign no one to interviews (qualified and unqualified alike) might make no one happy, but it is not unfairly discriminatory, since all individuals are treated similarly.

In this work, we adopt the perspective that given a suitable 
metric, solutions that are individually fair provide strong protections from discrimination, even though they might not be envy-free.
Armed with this perspective, we seek to relax the IF requirements to allow for a richer set of solutions, while still providing meaningful protections against discrimination.

\subsection{This Work: Preference-Informed 
Fairness}
\label{sec:intro:PIIF}

Building on the perspective from \cite{dwork2012fairness}, we propose and study the notion of \emph{preference-informed individual fairness} (PIIF). Our guiding principle is:

\begin{center}
\parbox{0.9\textwidth}{
\centering
\emph{Allocations that deviate from individual fairness may be considered fair,\\provided the deviations are in line with individuals' preferences.}
}
\end{center}

Before describing PIIF, we establish some notation.
We model a decision-maker's policy
$\pi$ as a mapping from individuals to allocations, i.e.,\ distributions over outcomes.
We assume that each individual $i$ has preferences over the possible allocations, where $p \succeq_i q$ denotes that $i$ (weakly) prefers allocation $p$ to allocation $q$.
To discuss notions of individual fairness, we assume that $D$ is a divergence where $D(p,q)$ measures some distance between two allocations $p,q$ (e.g.\ the total-variation distance), and $d$ is the task-specific metric where $d(i,j)$ specifies the similarity between individuals $i$ and $j$.

Using this notation, we can restate the notions of IF and EF as follows.
A policy $\pi$ is \emph{individually-fair} (IF) if for all pairs of individuals $i,j$, the Lipschitz condition $D(\pi(i),\pi(j)) \le d(i,j)$ is satisfied.\footnote{Throughout, we assume that $d$ and $D$ are scaled appropriately to be in the same ``units.''  That is, without loss of generality, we assume the relevant Lipschitz constant in the IF-style constraints is $1$.}
A policy $\pi$ is \emph{envy-free} (EF) if for all individuals $i$, for all other individuals $j$, $\pi(i) \succeq_i \pi(j)$.

\paragraph{Preference-informed individual fairness.}
As in both IF and EF,
PIIF establishes fairness by comparing the allocation of each individual $i$ to the allocation of every other individual $j$.
For each such comparison, PIIF requires that either $\pi(i)$ satisfies individual fairness with respect to $\pi(j)$ or $i$ prefers their allocation $\pi(i)$ over some alternative allocation that would have satisfied individual fairness with respect to $\pi(j)$.
More technically, for $\pi$ to be considered PIIF for each individual $i$, we require that for every other individual $j$ there exists some alternative allocation $p^{i;j}$ that $i$ could have received that
satisfies the IF Lipschitz condition with respect to $\pi(j)$ and where
$i$ (weakly) prefers their actual allocation $\pi(i)$ to the IF alternative $p^{i;j}$.

\begin{definition*}[PIIF]
A policy $\pi$ that maps individuals to allocations satisfies {Preference-Informed Individual Fairness} with respect to a divergence $D$, a similarity metric $d$, and individual preferences $\set{\succeq_i}$,  if for every individual $i$, for every other individual $j$,
there exists an alternative allocation 
$p^{i;j}$ such that:
\begin{itemize}
\item $p^{i;j}$ is individually fair w.r.t $\pi(j)$: \hspace{0.1in}
$D\left(p^{i;j}, \pi(j)\right) \le d(i,j).$
\item $i$ (weakly) prefers $\pi(i)$ over  $p^{i;j}$: \hspace{0.17in} $\pi(i) \succeq_i p^{i;j}.$
\end{itemize}
\end{definition*}

We emphasize that, in general, $p^{i;j} \neq p^{j;i}$; that is, the alternative chosen for $i$ with respect to $j$'s allocation need not be the same as that chosen for $j$ with respect to $i$.
Figure~\ref{fig:summary} provides a succinct summary of the definitions of IF, EF, and PIIF.

\begin{figure}[ht!]
\centering
\fbox{\parbox{0.75\textwidth}{
\centering

\vspace{8pt}

\begin{tabular}{c c c}
{\bf Individual Fairness (IF)}
&&
{\bf Envy-Freeness (EF)}\\
\emph{for every $i$, for every $j$:} &~~~~~~& \emph{for every $i$, for every $j$:}\\
$D(\pi(i),\pi(j)) \le d(i,j)$
&&$\pi(i) \succeq_i \pi(j)$
\end{tabular}

\vspace{11pt}

\begin{tabular}{c}
{\bf Preference-Informed Individual Fairness (PIIF)}\\
\emph{for every $i$, for every $j$, there exists $p^{i;j}$ s.t.}\\
$D(p^{i;j},\pi(j)) \le d(i,j)$\\
$\pi(i) \succeq_i p^{i;j}$
\end{tabular}

\vspace{4pt}

}
}

\caption{Summary of individual fairness notions}
\label{fig:summary}
\end{figure}

PIIF preserves the spirit of the core interpersonal fairness guarantee of IF: for each individual $i$, for every individual $j$ who is similar to $i$, either $i$'s outcome distribution is similar to $j$'s, or $i$ receives an even better (more-preferred) outcome distribution.
The main advantage of PIIF over IF is that it allows for a much richer solution space, which can lead to preferable outcomes for individuals.
Further, PIIF does not restrict the allocations unnecessarily; as in IF, the constraints only bind when a pairs of individuals are sufficiently similar according to the metric.
In other words, PIIF -- unlike EF -- permits solutions that may be disappointing to some individuals (i.e.\ where $i$ envies $j$) but should not be considered discriminatory (because $i$ and $j$ are substantially different according to the task at hand).

Referring back to the career expo example, we note that the allocation where the three qualified candidates $i,j,$ and $k$ (deterministically) interview with their preferred employer is PIIF.
To see this, consider $i$ comparing their outcome to those of $j$ and $k$ under such an interview assignment.
Comparing with $j$, $i$ prefers outcome $X$ to receiving outcome $Y$, which would satisfy the IF constraint with respect to $j$.
Similarly, she prefers $X$ to $Z$, which would satisfy the IF constraint with respect to $k$.
Indeed, since $i$ receives her preferred outcome, one can argue that there is no discrimination against $i$ in the allocation.
Similar reasoning applies to $j$ and to $k$.
In fact, the allocation where each individual deterministically receives their preferred outcome is always PIIF, a property we find desirable for a fairness definition.
Further, consider the allocation of $\ell$, who we assumed was significantly less qualified than $i$ (and thus, $j$ and $k$).
If $\ell$ is sufficiently dissimilar to all other candidates, then the scheduler can assign $\ell$ to any interview and still satisfy PIIF.
To see this, note that if $d(\ell,i)$ is sufficiently large, we can always take $p^{\ell;i} = \pi(\ell)$, and the constraints for individual $\ell$ with respect to $i$ will be satisfied
(with identical arguments when comparing $\ell$ to $j$ and $k$).

\subsection{Our Contributions}
Our 
running example illustrates that in many reasonable situations (involving rich and diverse individual preferences over outcomes), the existing notions of individual fairness and envy-freeness may not capture an appropriate notion of fairness or may unnecessarily constrain the decision-maker.
In high-stakes domains, such as employment and personalized content selection,
both limitations are significant and may hinder adoption of fairness-conscious decision-making.
We propose PIIF as a relaxation of IF that addresses the identified shortcomings of existing notions while still providing meaningful protections against discrimination.
We view this as an important conceptual contribution in its own right. 

With the motivation and definition for PIIF in place, we provide a comprehensive characterization of the relationship between PIIF and other individual notions of fairness.
In Section~\ref{sec:PIIF}, we show formally that PIIF can be viewed as a relaxation of both IF and EF; that is, any solution that satisfies either IF or EF also satisfies PIIF.
Further, we demonstrate that PIIF is a non-trivial relaxation of both notions, by proving that there exist settings in which PIIF solutions cannot be captured by IF or EF constraints alone, for \emph{any} choice of metrics $d$, $D$ and preferences.

To introduce PIIF, we have argued qualitatively that relaxing IF to PIIF allows for more preferable outcomes for individuals.
We quantify these claims by comparing the \emph{social welfare} of a decision-maker's policy achievable under PIIF and under IF.
In Section~\ref{sec:sw}, we show optimal bounds on the ratio of the best social welfare under PIIF to that under IF; the ratio can grow \emph{linearly} in the number of individuals classified or in the number of possible outcomes grows.

With the definition and properties of PIIF in place, we turn our attention to the algorithmic question of how to achieve PIIF.
In Section~\ref{sec:optimization},
we show that for a rich family of individual preferences, there is an efficient algorithm to minimize a convex objective subject to PIIF.
The result follows by observing that for structured classes of preferences, the set of PIIF constraints is convex.
In particular, to optimize over PIIF, we can augment the convex program defined for IF in \cite{dwork2012fairness} to capture the additional preference constraints.
As such, optimization subject to PIIF is only slightly more complex than optimization subject to IF.

Finally, we demonstrate the versatility of the PIIF framework, by applying preference-informed fairness in the context of targeted advertising (as studied by \cite{DworkI19}).
Recent empirical findings demonstrate that the ad allocation algorithms run by online advertising platforms may result in discrimination \cite{datta2015automated, lambrecht2018algorithmic, ali2019discrimination} and are thus facing legal scrutiny \cite{upturn2018, tobin2019pb, benner2019nyt}.
As such, developing formal frameworks for understanding fairness in such advertising systems is of great importance.
In Section~\ref{sec:ads}, we extend our definition of PIIF and our results to the multiple-task setting defined \cite{DworkI19} to model fairness desiderata for the domain of large-scale targeted advertising.
We show that in this practically-motivated setting, IF still may restrict the social welfare considerably compared to PIIF \emph{even when the individuals' similarity and preferences are perfectly aligned!}
The ratio of the best social welfare under PIIF to that of IF grows \emph{linearly} in the number of tasks.

\paragraph{Organization.}
Sections~\ref{sec:PIIF}-\ref{sec:ads} contain the technical details and proofs of our major contributions with some results deferred to the appendix.
We conclude in
Section~\ref{sec:discuss}
with comparisons to other related works and a discussion of the strengths and limitations of the current approach of preference-informed fairness
as well as
directions for future investigations.

\section{Preference-Informed Individual Fairness}
\label{sec:PIIF}

\paragraph{Preliminaries.}
Given a set of individuals $\X$,
we consider policies that assign every individual to an outcome in the set $\C$.
We allow randomized allocation rules $\pi:\X \to \Delta(\C)$, where
for each individual $i \in \X$, their allocation $\pi(i) \in \Delta(\C)$
represents a distribution over outcomes $c \in \C$.
We model individuals' preferences by assuming that every individual $ i\in\X$ has a reflexive and transitive binary relation $\succeq_{i}$ that encodes their preferences over allocations in $\Delta(\C)$; for $p,q\in\Delta(\C)$, we use $p\succeq_{i}q$ to denote that $i$ (weakly) prefers $p$ to $q$.\footnote{
$\succeq_{i}$ need not be \emph{total} nor \emph{antisymmetric} over $\Delta(\C)$.
}
We use $\succeq$ to denote the set of individuals' preference relations, $\succeq = \left\{ \succeq_{i}\right\}
_{i \in \X}$.

One important structured class of preference relations are those that admit a \emph{utility function}.
Here, we assume each individual $i \in \X$ has a real-valued function over allocations $u_i:\Delta(\C) \to \R$, where
$u_i(\pi(i))$ represents the utility to individual $i$ from the allocation given by $\pi$.
Given such a utility function, $p \succeq_i q$ if and only if $u_i(p) \ge u_i(q)$.

With this technical notation in place, for completeness, we restate the three definitions of fairness.
\begin{definition}[Individual Fairness]
Given a divergence $D:\Delta(\C) \times \Delta(\C) \to [0,1]$ and a similarity metric  $d:\X \times \X \to [0,1]$,
a policy $\pi:\X \to \Delta(\C)$ is $(D, d)$-individually fair if for every two individuals $i,j \in \X\times \X$, the following Lipschitz condition holds.
\begin{equation}
D(\pi(i),\pi(j)) \le d(i,j) \label{eqn:lipschitz}
\end{equation}
\end{definition}
\begin{definition}[Envy Freeness]
Given a set of preferences $\succeq$,
a policy $\pi:\X \to \Delta(\C)$
is $\succeq$-envy-free if
for all individuals $i \in \X$, and for all other individuals $j \in \X$,
\begin{equation}
\label{eqn:ef}
    \pi(i)\succeq_{i}\pi(i)
\end{equation}
\end{definition}
\begin{definition}[Preference-Informed Individual Fairness]
\label{def:PIIF}
Given a divergence $D:\Delta(\C) \times \Delta(\C) \to [0,1]$, a similarity metric  $d:\X \times \X \to [0,1]$, and a set of preferences $\succeq$,
a policy $\pi:\X \to \Delta(\C)$
is $(D,d,\preceq)$-PIIF if
for all individuals $i \in \X$, for all other individuals
$j \in \X$, there exists an allocation $p^{i;j} \in \Delta(\C)$
such that:
\begin{gather}
D\left(p^{i;j}, \pi(j)\right) \le d(i,j)\label{def:PIIF:if}\\
\pi(i) \succeq_i p^{i;j}\label{def:PIIF:pref}
\end{gather}
\end{definition}
Often, the divergence $D$, metric $d$, and preferences $\succeq$ will be fixed.
In these contexts, we use $\Pi^{\IF}, \Pi^{\EF}, \Pi^{\PIIF}$ to denote the set of IF, EF, and PIIF solutions, respectively.

\subsection{PIIF relaxes IF and EF}
We have argued informally that PIIF captures the appealing aspects of both IF (strong discrimination protections) and EF (respecting the preferences of individuals) without being overly prescriptive in a way that might hurt individuals or the decision-maker.
Our first result formalizes these claims, by characterizing PIIF as a relaxation of both IF and EF.
We show that any policy that is either IF or EF is also PIIF.
\begin{proposition}
Fixing a divergence, the metric, and preferences, $\Pi^{\IF} \subseteq \Pi^{\PIIF}$  and  $\Pi^{\EF} \subseteq \Pi^{\PIIF}$.
\end{proposition}
As solution concepts, both IF and EF are always feasible, but for very different reasons: for IF, any allocation that treats all individuals identically is feasible; for EF, the allocation that gives everyone their most-preferred outcome is envy-free.
Thus, both of these extreme solutions will also be feasible for PIIF.
In general, PIIF will be a 
strict relaxation of these concepts that allows for interpolation between the notions.
Intuitively, more diverse preferences of individuals tend to give rise to richer sets of PIIF solutions compared to IF, and nontrivial metrics $d$ (i.e.,\ further from the all-zeros ``metric'') give rise to richer sets of PIIF solutions compared to EF.
Given the right framing, the proof of this result is almost immediate.
\begin{proof}
To see that an IF policy $\pi$ satisfies PIIF, for each $i$, we take $p^{i;j} = \pi(i)$ for all $j$.
Consider an allocation $\pi \in \Pi^{\IF}$.
From the perspective of any individual $i\in \X$,
when comparing to individual $j \in \X$,
if $p^{i;j} = \pi(i)$,
then, by the fact that
$\pi$ satisfies IF, condition (\ref{def:PIIF:if})
is satisfied.
By reflexivity of $\succeq_i$,
(\ref{def:PIIF:pref}) is also satisfied, so
$\pi \in \Pi^{\PIIF}$.

To see that an EF policy $\pi$ satisfies PIIF, for each $i$, we take $p^{i;j} = \pi(j)$ for all $j$.
Consider an allocation $\pi \in \Pi^{\EF}$.
From the perspective of any $i\in \X$,
when comparing to  $j \in \X$,
if $p^{i;j} = \pi(j)$,
then, condition (\ref{def:PIIF:if})
is satisfied trivially because $D(\pi(j),\pi(j)) = 0$.
Since $\pi$ satisfies EF, we know that
$\pi(j) \preceq_i \pi(i)$, so condition (\ref{def:PIIF:pref})
also holds; thus $\pi \in \Pi^{\PIIF}$.
\end{proof}

\paragraph{PIIF generalizes IF and EF.}
We remark that this intuition also shows that PIIF is a \emph{generalization} of both IF and EF;
that is, both notions can be ``implemented'' as special cases of PIIF.
To implement IF, we can set all individual's preference relation $\succeq_i$ to be the trivial reflexive relation, where for all allocations $p$, $p \succeq_i p$, and for all nontrivial pairs $p \neq q$, $p$ and $q$ are incomparable.
To implement EF, we simply take $d(i,j) = 0$ for all $i,j$ pairs.
In other words, we can think of the set of IF solutions
as those where we require the alternative allocation for $i$
compared to $j$ to be $i$'s actual allocation
$p^{i;j} = \pi(i)$, and we can think of the set of
EF solutions as those where we require the alternative
allocation for $i$ compared to $j$ to be $j$'s allocation $p^{i;j} = \pi(j)$.

\paragraph{PIIF is a meaningful relaxation of IF and EF.}

A natural question to ask is whether we need to introduce a new definition of individual fairness.
In particular, we might hope that we could ``implement'' PIIF using IF with a metric that incorporates preferences or with EF with preferences that incorporate distances.
We argue that when there is a rich set of possible outcomes and a correspondingly-rich set of possible preferences,
such an approach is infeasible.
In particular, PIIF captures constraints that could not be cast within the language of IF or EF alone.

To build intuition, we revisit the career expo example: suppose that two similarly qualified individuals $i$ and $j$ have a similar top choice (say, $X$), but disagree on their second choice ($i$ prefers $Y$, whereas $j$ prefers $Z$).
Do these individuals have similar preferences or divergent ones? Intuitively, a fair assignment could give them similar probabilities of seeing $X$, but different probabilities of seeing $Y$ and $Z$.
Individual fairness treats all outcomes symmetrically for all individuals, and does not let us make such distinctions. 
The following proposition strengthens this intuition, demonstrating that there are in fact settings in which EF preferences cannot be encoded using any IF metric, and vice versa.
Note that this implies that PIIF -- a relaxation of both notions -- captures constraints that cannot be cast within the language of IF or EF alone. 

\begin{proposition}
\label{prop:IF_EF_differences}
There exists a set of preferences $\succeq$ such that for any choice of divergence $D$ and metric $d$ $$\Pi^{\succeq\textrm{-}\EF} \neq \Pi^{(D,d)\textrm{-}\IF}.$$
There exists a divergence-metric pair $D,d$ such that for any choice of preferences $\succeq$,
$$\Pi^{(D,d)\textrm{-}\IF} \neq \Pi^{\succeq\textrm{-}\EF}.$$
\end{proposition}
\begin{proof}
In both constructions, we will assume there are two disjoint groups of individuals $S,T \subseteq \X$.
Consider two outcomes $p,q \in \C$.
Suppose $\succeq$ is such that for some $i \in S$ and $j \in T$,
$p \succ_i q$ and $q \succ_j p$.
Consider any $D$ and $d$:
if $D(p,q) \le d(j,i)$, then assigning $p$ to $j$ and $q$ to $i$ will be $(D,d)$-IF, but it is not $\succeq$-EF;
otherwise, if $D(p,q) > d(i,j)$, then assigning $p$ to $S$ and $q$ to $T$ will not be $(D,d)$-IF, even though it is $\succeq$-EF.
Thus, no $D,d$ can capture $\succeq$-EF.

Now take $D$ to be total variation distance and consider a metric $d$ where $d(i,j) = 0$ for $i,j \in S\times S$ and $T \times T$, and $d(i,j) = 1$ for $i,j \in S \times T$.
Under this metric, assigning any fixed allocation to everyone in $S$ and any (potentially-different) fixed allocation to everyone in $T$ is $(D,d)$-IF.
Consider some $\succeq$.
If there is some $i \in S$ such that $p \succ_i q$ or if there is some $j \in T$ such that $q \succ_j p$, then the $(D,d)$-IF allocation that assigns $q$ to every $i \in S$ and $p$ to every $j \in T$ is not $\succeq$-EF.

Thus, for all individuals in $i \in S \cup T$, $\succeq_i$ must be either the relation, where $p \equiv q$ or the trivial reflexive relation where $p$ and $q$ are incomparable.
Suppose $i,j \in S \times S$ both have $\succeq_i = \succeq_j = \equiv$.
Then, the solution that assigns $p$ to $i$ and $q$ to $j$ is $\succeq$-EF, but violates the Lipschitz condition of $(D,d)$-IF.
On the other hand, if there is some $i \in S$ that holds the trivial reflexive relation, then the $(D,d)$-IF solution that assigns $p$ to all of $S$ and $q$ to all of $T$ will not satisfy $\succeq$-EF, because $p \not \succeq_i q$.
\end{proof}

\subsection{Metric Envy-Freeness}

\label{sec:PIIF:mef}

We arrived at PIIF by starting with the metric-based IF as a strong notion of nondiscrimination and relaxing the notion to incorporate individuals' preferences and allow for a richer set of solutions while providing a meaningful protections against discrimination.
A conceptually-different approach towards these goals would start with preference-based EF, but allow the decision-maker some freedom by incorporating distances between individuals.
In particular, consider the following relaxation of EF, which we call \emph{Metric Envy-Freeness} (MEF), that intuitively captures the idea that no individual should envy the allocation of any other \emph{similar} individual.
\begin{definition}
Suppose each individual $i \in \X$ has a utility function $u_i:\Delta(\C) \to \R$; let~$\U = \set{u_i}_{i \in \X}$.
Given a similarity metric $d:\X \times \X \to \R^+$, a policy $\pi: X \to \Delta(\C)$ satisfies $(d,\U)$-metric-envy-freeness if for every individual $i \in \X$, for every other individual $j \in \X$,
\begin{equation*}
    u_i(\pi(i)) \geq u_i(\pi(j)) - d(i,j)
\end{equation*}
\end{definition}
This definition starts with the envy-freeness constraint for utility-based preferences, but then relaxes the constraint between $i$ and $j$ by their distance according to the metric.
For the metric-utility comparison of MEF to be meaningful, we assume that utilities and metric distances are normalized to one another;
without loss of generality, assume that each utility and metric distance is bounded in $[0,1]$.
For each pair $i,j$, the notion interpolates between two extremes based on the value of $d(i,j)$: if $d(i,j) = 0$, then envy-freeness binds; when $d(i,j) = 1$, the allocation $i$ receives is not constrained by the allocation $j$ receives.

As $d(i,j) \ge 0$ for all pairs of individuals, MEF is clearly a relaxation of EF.
That said, it's not immediately obvious how MEF relates to IF or PIIF.
While conceptually different, we show that MEF captures a closely-related notion of fairness to PIIF, in the special case where preferences are given by structured utility functions.
To relate MEF to PIIF, we need to assume the following Lipschitz conditions.
\begin{assumption}[Lipschitz utility]
A utility function $u:\Delta(\C) \to \R$ is $\ell$-Lipschitz with respect to $D:\Delta(\C) \times \Delta(\C) \to \R^+$ if
\begin{equation*}
\card{u(p) - u(q)} \le \ell \cdot D(p,q).
\end{equation*}
\end{assumption}
Lipschitz utility functions are quite natural. For instance, taking $D$ to be the total variation distance, if individuals' preferences admit an expected utility function, where each outcome has utility in $[0,1]$, then individuals' utilities will be $1$-Lipschitz.
In other words, individuals' utilities are not highly sensitive to very small changes in the allocation they receive.
\begin{assumption}[Reverse-Lipschitz utility]
A utility function $u:\Delta(\C) \to \R$ is $\ell$-reverse-Lipschitz with respect to $D:\Delta(\C) \times \Delta(\C) \to \R^+$ if
\begin{equation*}
\frac{1}{\ell} \cdot D(p,q) \le \card{u(p) - u(q)}.
\end{equation*}
\end{assumption}
Reverse-Lipschitz utility functions are less natural.
This assumption implies that no pair of outcomes is valued very similarly.
One natural setting where the reverse-Lipschitz condition holds nontrivially is in the case of binary outcomes, where each individual prefers one outcome over the other.
Under these assumptions, we can show the following relationship between MEF and PIIF.
\begin{theorem} \label{prop:mef}
Suppose $D:\Delta(\C) \times \Delta(\C) \to \R^+$ is a divergence, $d:\X \times \X \to \R^+$ is a similarity metric, and $\U = \set{u_i}$ is a family of utility functions.
Let $\succeq^\U$ denote the family of preferences induced by $\U$.
Consider a policy $\pi:\X \to \Delta(\C)$.
For some constant $\ell \ge 1$:
\begin{itemize}
\item Suppose for all $i \in \X$, $u_i$ is $\ell$-Lipschitz with respect to $D$.
Then, $$\Pi^{(D,d,\succeq^\U)\textrm{-}\PIIF} \subseteq \Pi^{(\ell \cdot d, \U)\textrm{-}\mathrm{MEF}}.$$
\item Suppose for all $i \in \X$, $u_i$ is $\ell$-reverse-Lipschitz with respect to $D$.
Then, $$\Pi^{(d, \U)\textrm{-}\mathrm{MEF}} \subseteq \Pi^{(D, \ell \cdot d, \succeq^\U)\textrm{-}\PIIF}.$$
\end{itemize}
\end{theorem}

\begin{proof}
To see that PIIF implies MEF, we start with a policy $\pi$ that satisfies $(D,d,\succeq^\U)$-PIIF.
To establish MEF, we compare the utility of individual $i$ on their allocation $\pi(i)$ to that of another individual $\pi(j)$; we denote by $p^{i;j}$ the alternative allocation for $i$ that satisfies the PIIF constraints.
\begin{align}
u_{i}(\pi(i)) &\geq u_{i}(p^{i;j}) \label{mef:fwd:piif1}\\
&\geq u_{i}(\pi(j))-\left[u_{i}(\pi(j))-u_{i}(p^{i;j})\right] \notag\\
&\geq u_{i}(\pi(j))-\ell \cdot D\left(\pi(j),p^{i;j}\right) \label{mef:fwd:lipschitz}\\
&\geq u_{i}(\pi(j))- \ell \cdot d(i,j)\label{mef:fwd:piif2}
\end{align}
where (\ref{mef:fwd:piif1}) follows by the fact that $\pi$ satisfies PIIF; (\ref{mef:fwd:lipschitz}) follows by the Lipschitz condition;
and (\ref{mef:fwd:piif2}) follows again from the fact that $\pi$ satisfies PIIF.
Thus,
$
u_i(\pi(i)) \ge u_i(\pi(j)) - \ell \cdot d(i,j)
,$ 
so $\pi$ is $(\ell \cdot d, \U)$-MEF.

To see that MEF implies PIIF, we start with a policy $\pi$ that satisfies $(d, \U)$-MEF.
To establish PIIF, we consider an arbitrary pair of individuals $i$ and $j$, and exhibit an allocation $p^{i;j}$ that satisfies the PIIF conditions with respect to $\pi(i)$ and $\pi(j)$.
Comparing individual $i$ to individual $j$, we consider two cases.

First, suppose $u_i(\pi(i)) \ge u_i(\pi(j))$ and take $p^{i;j} = \pi(j)$.
In this case, the Lipshitz constraint is trivially satisfied,
\begin{gather*}
D(p^{i,j}, \pi(j)) = D(\pi(j),\pi(j)) = 0 \le \ell \cdot d(i,j)
\end{gather*}
and the assumption that  $u_i(\pi(i)) \ge u_i(\pi(j))$ implies that
\begin{gather*}
\pi(i) \succeq_i \pi(j) = p^{i;j},
\end{gather*}
so the PIIF constraints are satisfied.

Next, suppose $u_i(\pi(i)) < u_i(\pi(j))$ and take $p^{i;j} = \pi(i)$.
In this case, the preference condition is trivially satisfied,
\begin{gather*}
\pi(i) \succeq_i \pi(i) = p^{i;j}
\end{gather*}
and the Lipschitz condition follows as
\begin{align}
D(p^{i;j}, \pi(j)) &= D(\pi(i),\pi(j)) \notag \\
&\le \ell \cdot \left(u_i(\pi(j)) - u_i(\pi(i))\right)\label{mef:bwd:bilipschitz}\\
&\le \ell \cdot d(i,j) \label{mef:bwd:mef}
\end{align}
where (\ref{mef:bwd:bilipschitz}) follows by the reverse-Lipschitz condition and (\ref{mef:bwd:mef}) follows by the fact that $\pi$ satisfies MEF.
Thus, $\pi$ is $(D,\ell \cdot d, \succeq^\U)$-PIIF.
\end{proof}

\paragraph{Relating PIIF and MEF.}
Theorem~\ref{prop:mef} shows that under appropriate assumptions, we can relate the notions of PIIF and MEF.
We interpret the theorem to say that, in most settings we would apply preference-informed fairness, MEF is a strictly more relaxed notion than PIIF; that is, every PIIF solution will be MEF, but there will be MEF solutions that do not satisfy PIIF.

Specifically, as we remarked earlier, it is reasonable to assume that individuals' utility functions are $1$-Lipschitz with respect to $D$.
In this case, small changes in the allocation according to $D$ cannot result in dramatically different utilities, and Theorem~\ref{prop:mef} says that $(D,d,\succeq^\U)$-PIIF implies $(d,\U)$-MEF.

In general, with a rich set of outcomes, we do not expect that individuals' utility functions will be reverse-Lipschitz with respect to $D$.
As such, there are cases when $(d,\U)$-MEF will be considerably more relaxed than $(D,d,\succeq^\U)$-PIIF, allowing for deviations from $(D,d)$-IF that do not give utility improvements for all individuals.
To see this point, consider the following example with two individuals $i,j$, where $d(i,j) = 0.5$ and two outcomes $p,q$, with the utility functions of $i$ and $j$ defined as follows.

\begin{center}
\begin{tabular}{c|c|c}
 & $p$ & $q$  \\ \hline
$u_{i}$ & $1.0$ & $0.5$ \\
$u_{j}$ & $0.5$ & $1.0$
\end{tabular}
\end{center}

We take $D$ to be the total variation distance between allocations.
It's easy to verify that all allocations that treat the individuals identically and the welfare-maximizing solution where $i$ receives $p$ and $j$ receives $q$ will be MEF;
this can be seen by noting that each of these solutions satisfies EF, thus, also MEF.
In fact, in this example, the utility functions are sufficiently Lipschitz to guarantee that all PIIF solutions will be MEF.

On the other hand, consider the allocation where $i$ receives $q$ and $j$ receives $p$, deterministically.
This solution is not IF, because $1 = D(q,p) > d(i,j) = 0.5$, nor EF, because both individuals envy the others' solution.
Further, the solution is not PIIF.
To see this, note that because each individual receives their least favorite outcome, the only surrogate $p^{i;j}$ that can be chosen for individual $i$ such that $q \succeq_i p^{i;j}$ is $p^{i;j} = q$ (similarly, $p^{j;i} = p$ for individual $j$).
As the actual allocation is not IF, one of the PIIF constraints will be violated.

Still, this allocation does satisfy MEF.
Specifically,
$0.5 = u_i(q) \ge u_i(p) - d(i,j) = 1.0 - 0.5$ and
$0.5 = u_j(p) \ge u_j(q) - d(j,i) = 1.0 - 0.5$.
In other words, in this instance, MEF allows for a solution that is not permitted by any of the other notions of individual fairness we consider.
In particular, the allocation deviates from individual fairness in a way that does not help either individual.
This example suggests that in reasonable settings, MEF is a strictly weaker concept than PIIF and may be too permissive.

\section{Welfare under IF, EF, and PIIF}
\label{sec:sw}

In order to motivate PIIF, we argued that IF may be overly-restrictive and limit the ``quality'' of solutions from the perspective of individuals.
In this section, we formalize this claim, showing that PIIF admits solutions that can be significantly more preferable to individuals.
We quantify the idea of individual quality using the game-theoretic notion of social welfare.

We restrict our attention to the common setting where individuals hold preference relations that admit a utility function.
In this setting, given a policy $\pi$, we can track the \emph{social welfare} of the policy, defined to be the overall utility experienced by individuals.

\begin{definition}
Given a policy $\pi:\X \to \Delta(\C)$, 
social welfare $\W(\pi)$ is the sum of the individuals' utilities under $\pi$.
\begin{equation}
\W(\pi) = \sum_{i\in \X}u_{i}(\pi(i))
\end{equation}
For a collection of allocations $\Pi$,
we let $\W^*(\Pi) = \max_{\pi \in \Pi} \W(\pi)$ denote the optimal social welfare achievable by any allocation in $\Pi$, and let $\W^* = \W^*(\Delta(\C)^\X)$ denote the optimal (unconstrained) social welfare.
\end{definition}

An important special case of preferences that admit a utility function are those that admit an \emph{expected utility function}.
Using such preferences is a standard approach in economics for modeling decision-making in the presence of uncertainty.\footnote{Although not all preference relations admit this form, a rich class of preferences do. For example, different levels of tolerance towards risk can be captured within this framework (e.g., a risk-averse individual would have a utility function $u$ which is concave). Von Neumann and Morgensterm \cite{von2007theory} provide a complete characterization of this class of preference relations.}
\begin{definition}[Expected utility representation]
\label{def:pref:eut}
A preference relation $\succeq$ admits an expected utility representation if and only if there exists a function $u:\C\rightarrow\mathbb{R}^+$, such that for any two allocations $p, q \in\Delta(C)$, 
\begin{equation}
    p \succeq q \iff\sum_{c\in C}p_c\cdot u(c)\geq \sum_{c\in C} q_c \cdot u(c)
\end{equation}
\end{definition}
\paragraph{Individual fairness may restrict social welfare.}
Using the notation introduced above, note that $\W^*(\Pi^\PIIF) = \W^*$ because the welfare-maximizing allocation is feasible for PIIF.
Here, we aim to understand how much IF may restrict the best social welfare compared to PIIF, by relating $\W^*(\Pi^\IF)$ to $\W^*$.

Intuitively, social welfare can be hurt significantly by requiring IF compared to PIIF when many individuals are considered similar, but there is a diversity of preferences over outcomes.
Formalizing this intuition, we argue that without any assumptions about the class of utility functions of individuals, the ratio between the best social welfare under IF and PIIF can grow with the number of individuals.
Further, even under the stronger assumption that individuals' preferences admit an expected utility representation, the ratio can grow with the number of outcomes.
\begin{theorem}\label{prop:single:gap}
There exists a family of instances such that $$\frac{\W^*}{\W^*\left(\Pi^\IF\right)} \geq \card{\X}.$$
Additionally, there exists a family of instances where individuals' preferences admit an expected utility representation and
$$\frac{\W^*}{\W^*\left(\Pi^\IF\right)} \geq \card{\C}.$$

\end{theorem}
\begin{proof}
The two claims of the theorem follow by similar constructions.
Suppose every individual is considered
similar according to $d$; that is, for all $i,j \in \X\times \X$, $d(i,j)=0$.
This means that any IF solution must assign every individual the same distribution over outcomes.
To show the gaps, we will compare the best IF solution (i.e.\ constant allocation) to the welfare-maximizing solution, which is feasible under PIIF.

To begin, suppose we allow individuals to specify arbitrary utility functions;
let each individual $i \in \X$ hold a distinct $p_i \in \Delta(\C)$ (i.e., $p_i \neq p_j$ for all $i \neq j$) such that $u_i(p_i) = 1$ and $u_i(q) = 0$ for any $q \neq p_i$.
In this case, the optimal social welfare is $\W^* = \card{\X}$.
For any fixed allocation, however, the best social welfare is to choose a distribution over the set of $\set{p_i}$.
Any such distribution will achieve welfare $1$; thus, $\W^*(\Pi^\IF) = 1$.

Now, suppose every individual is required to specify an expected utility representation.
Let each individual choose some $c \in \C$, such that a $1/\card{\C}$-fraction of individuals prefer each outcome $c$; let $u_i(c) = 1$ for their preferred outcome and $u_i(c) = 0$, otherwise.
Again, the optimal social welfare is $\W^* = \card{\X}$.
Under any policy that assigns every individual the same fixed allocation $p$, the social welfare is given by
$\sum_{i \in \X} \sum_{c \in \C} p_c \cdot u_i(c) = \frac{\card{\X}}{\card{\C}} \cdot \sum_{c \in \C} p_c  = \frac{\card{\X}}{\card{\C}}$.
Thus, $\W^*(\Pi^\IF) = \frac{\card{\X}}{\card{\C}}$.
\end{proof}

We note that the gaps demonstrated in Theorem~\ref{prop:single:gap} are optimal in their settings.
In particular, any constant allocation will be IF, and we can always recoup a $1/\card{\X}$ fraction of the social welfare with a constant allocation tailored for the individual with the highest utility.
Further, in the case where preferences admit an expected utility representation, we can choose the constant allocation on the $c \in \C$ of maximum welfare.
Finally, we note that one unsatisfying aspect of these constructions is that they rely on the fact that all individuals are similar according to $d$.
In Section~\ref{sec:ads:gap}, we show that in the multiple-task setting of \cite{DworkI19}, such gaps exist even with nontrivial metrics that seem to be aligned with social welfare.

\paragraph{PIIF does not guarantee social welfare.}
Because PIIF is a relaxation of IF, the best social
welfare achievable under PIIF is always at least that of IF.
That said, because PIIF is a strict relaxation of IF, it does not necessarily guarantee that \emph{every} allocation's social welfare improves under PIIF.
In particular, when the decision-maker seeks to optimize a utility function that runs against individuals' utilities within the set of PIIF solution, the obtained social welfare may be arbitrarily worse under the PIIF constraints than IF constraints.

Suppose that the decision-maker has an additive utility function of the  form $f(\pi) = \sum_{i \in \X} f_i(\pi(i))$, for $f_i:\Delta(\C) \to \R$.
Let $\pi_f^\IF = \argmax_{\pi \in \Pi^\IF} f(\pi)$  and $\pi_f^\PIIF = \argmax_{\pi \in \Pi^\PIIF} f(\pi)$ denote the optimal IF (resp., PIIF) solution in terms of $f(\cdot)$.

\begin{proposition}
There exists a family of instances and an additive utility function $f$ such that $$\W(\pi_f^\PIIF)= 0 < \W(\pi_f^\IF).$$
\end{proposition}
\begin{proof}
Suppose there are two disjoint classes of individuals, $S$ and $T$, which each make up half of $\X$, and all individuals are similar; for all $i,j \in \X \times \X$, $d(i,j) = 0$.
Suppose there are three outcomes $\C = \set{p,q,r}$.
Consider the utility functions defined as follows.

\begin{center}
\begin{tabular}{c|c|c|c}
 & $p$ & $q$ & $r$ \\ \hline
$f_{i \in S}$ & $1/2+\eps$ & ~~~$1$~~~ & ~~~$0$~~~ \\
$f_{j \in T}$ & $1/2+\eps$ & $0$ & $1$\\ \hline
$u_{i \in S}$ & $1$ & $0$ & $0$ \\
$u_{j \in T}$ & $1$ & $0$ & $0$
\end{tabular}
\end{center}

Under IF, the decision-maker must treat all individuals identically.
Given this constraint, the outcome that maximizes $f$ is allocating $p$ deterministically to everyone.
This allocation $\pi_f^\IF$ achieves $\W(\pi_f^\IF) = 1$.
Without the constraint that all individuals' allocations are identical, the allocation that assigns $q$ to individuals from $S$ and $r$ to individuals from $T$ maximizes $f$.
In fact, this allocation will be feasible for PIIF: note that every individual experiences $0$ utility from everyone's allocation, so no one envies anyone else; under PIIF, envy-free solutions are feasible.
Thus, $\W(\pi_f^\PIIF) = 0$. In fact, this reveals that the proposition holds not only for PIIF, but also for any notion that relaxes EF.
\end{proof}

This construction demonstrates that the PIIF constraints alone do not guarantee improved social welfare compared to the IF constraints. We remark, however, that this is in line with our initial motivation for PIIF: decoupling the objective of provably preventing discrimination from the objective of ensuring beneficial outcomes in aggregate.  Finally, we note that if this is a concern, it can be addressed within our framework by adding a constraint to the optimization program, discussed in detail in Section~\ref{sec:optimization}, that ensures the social welfare is above some baseline.
In particular, an appropriate individually fair solution could act as this baseline, by first computing the social welfare obtained by it and then requiring that the resulting PIIF solution has at least this social welfare. At the extreme, we could even add such a constraint on the utility experienced by each individual; thus, obtaining the guarantee that any deviations from an IF solution are optimal, from the individuals' perspective, compared to some benchmark IF solution.

\section{Optimization subject to PIIF}
\label{sec:optimization}

As we have argued, satisfying PIIF is always feasible:
on the one hand, we can take any IF solution, including a trivial policy that treats all individuals identically; alternatively, we can take any EF solution, including the welfare-maximizing policy that gives everyone their most-preferred allocation.
In this section, we study the question of efficient optimization of a decision-maker's utility function subject to PIIF constraints.
As is standard in much of learning and optimization, we frame this task as the following minimization problem:
\begin{equation*}
\label{eqn:PIIF_program}
\begin{aligned}
& \underset{ \pi: \X \to \Delta(\C)}{\text{minimize}}
& & f(\pi) \\
& \text{subject to}
& & \pi \in \Pi^{\text{PIIF}}
\end{aligned}
\end{equation*}
In this section, we answer the question of feasibility of efficient optimization in the positive when $f$ is convex, and the preferences arise from a structured, but rich class of relations.

\subsection{Structured preferences}

In principle, PIIF can be instantiated with any notion of preference.
Without assuming anything about the preferences, however, the PIIF constraints could be difficult to handle: the space of allocations, over which the PIIF constraints are defined, is exponential.
In realistic settings, where the number of individuals or outcomes is large, this exponential dependence may be intractable. 
Towards efficient optimization, we focus on two rich and structured preference classes.

First, we include the prominent class of preferences that admit an expected utility representation, as defined in Definition~\ref{def:pref:eut}.
Additionally, we include the class of \emph{stochastic domination} preferences.
Stochastic domination formalizes the intuition that for any distribution over outcomes, a shift of probability mass from less desirable outcomes to more desirable outcomes is considered preferable.
Viewing an allocation $p \in \Delta(\C)$ as a discrete probability distribution, we denote by $c \sim p$ an outcome randomly sampled from $p$.
\begin{definition}[Stochastic domination]
\label{def:pref:dom}
For an individual with a utility function $u:\C \rightarrow [0,M]$ and for any two allocations $p, q \in\Delta(C)$,
$p$ stochastically dominates $q$ if
\begin{equation*} 
    p \succeq q \iff\forall x \in [0,M], \quad \Pr_{c\sim p}\left[u(c)\geq x\right]\geq \Pr_{c\sim q}\left[u(c)\geq x\right]
\end{equation*}
\end{definition}
That is, an allocation $p$ is (weakly) preferred over $q$ if for every possible level of utility $x$, the probability of achieving at least $x$ is no worse under $p$ than it is under $q$.
Note that stochastic domination represents an interesting example of a non-total preferences, as two allocations may be incomparable.\footnote{We remark that this preference notion is a special case of the statistical concept of \textit{first-order} stochastic domination \cite{hadar1969rules, bawa1975optimal}.}

\subsection{Efficient optimization subject to PIIF}
Here, we prove that when individuals' preferences are of the forms defined above, the PIIF constraints admit efficient optimization.
Formally, the following theorem demonstrates that when the divergence over allocations $D$ is taken to be
total variation distance $D_\mathrm{tv}$, and assuming oracle access to the
individual-fairness metric $d$, 
we can write the PIIF constraints as a set of (polynomially-many) linear inequalities; thus,
we can efficiently minimize any convex objective $f$.
\begin{theorem}
\label{thm:optimization}
Let $\succeq = \left\{ \succeq_{i}\right\} _{i\in \X}$ be the set of individuals' preferences. If every $\succeq_{i}$ is either the stochastic domination relation or admits an expected utility representation, then the set of $(D_\mathrm{tv},d,\succeq)$-PIIF allocations forms a convex polytope in $\mathbb{R}^{k}$, where 
$k = \poly\left(\card{\X},\card{\C}\right)$.
\end{theorem} 
\begin{proof}
We specify the PIIF constraints using the following
variables:  for all $i \in \X$, let $\pi(i) \in \Delta(\C)$ be a vector
denoting the actual allocation;
for every pair of individuals $(i,j) \in \X\times \X$,
let $p^{i;j} \in \Delta(\C)$ be a vector 
denoting the alternative allocation for $i$ when comparing to $j$.
We argue that the PIIF constraints given in (\ref{def:PIIF:if}) and (\ref{def:PIIF:pref}) can each be written as linear inequalities over
these variables.

First, since $D$ is taken to be the total variation distance we can
translate (\ref{def:PIIF:if}) as $\frac{1}{2} \cdot \sum_{c \in \C}\card{p^{i;j}_c - \pi(j)_c} \le d(i,j)$.
This can be written as $2\cdot \card{\C}+1$ linear inequalities
(with the introduction of $\card{\C}$ additional variables representing
the absolute values).

Next, we turn to the constraint given in (\ref{def:PIIF:pref}). First, consider the case of the preference relations admitting an expected utility form. Let $u_i$ be the utility  function for individual $i$.
By definition,
\begin{equation*}
    \pi(i)\succeq_i p^{i;j} \iff
    \sum_{c\in C}\pi(i)_{c}\cdot u_{i}(c) \ge
    \sum_{c\in C}p^{i;j}_{c}\cdot u_{i}(c).
\end{equation*}
Thus, for every $i \in \X$, the PIIF constraint given in (\ref{def:PIIF:pref}) with respect to $j \in \X$ 
can be written as a linear inequality
in the variables $p^{i;j}$ and $\pi(i)$.

Next, we consider the case of the stochastic domination preference relation. We introduce some notation as follows. Fix an individual $i$ and their allocation, $\pi(i)$.
Suppose $\card{\C} = k$, and that the outcomes in $\C$ are labeled in decreasing order according to $i$'s preferences: $u_{i}(c_{0})\geq u_{i}(c_{1})\geq\dots\geq u_{i}(c_{k-1})$. With this ordering in place, we have that for any allocation $p \in \Delta(\C)$ and every rank $r \in [k]$,
$ 
    \Pr_{c\sim p}\left[u_{i}(c)\geq u_{i}(c_r)\right] = \sum_{t=1}^{r}p_{t}.
$ 
Thus, for each $i \in \X$ we can write the stochastic domination condition as $k$ linear inequalities for each $j \in \X$, where
\begin{equation*}
\pi(i)\succeq_i p^{i;j} \iff
\forall r \in [k]:~
\sum_{t \in [r]}\pi(i)_{t} \ge \sum_{t \in [r]} p^{i;j}_t.
\end{equation*}
Importantly, this demonstrates that for this preference relation, the constraint given in (\ref{def:PIIF:pref}) can be enforced using an additional $O\left(\left|\C\right|\right)$ linear constraints, one for every $r\in\left[k\right]$.
\end{proof}

\paragraph{Other notions of preference}
Theorem \ref{thm:optimization} focuses on the case in which individuals' preferences satisfy one of the two forms discussed above and formalized in Definitions \ref{def:pref:eut} and \ref{def:pref:dom}.
Naturally, however, not all preference relations satisfy one of these two forms.
Appealing examples include preferences where the individual deems some of the outcomes to be substitutes (i.e.,\ interested in exactly one) or complements (i.e.,\ only interested in the complete set) or possibly preferences that value diversity of outcomes.
We leave the question of whether PIIF admits efficient optimization over such non-convex preferences as an interesting direction for future research. 
\section{Fairness in Targeted Advertising: Multiple-Task PIIF}
\label{sec:ads}

In this section, we extend the definition and study of preference-informed individual fairness to the \textit{multiple-task} setting, formalized and studied by Dwork and Ilvento \cite{DworkI19}.
This setting was introduced as a model in which to study fairness in targeted advertising, a form of online advertising where ad platforms allow advertisers to specify the characteristics of users they would like to reach, and then make algorithmic decisions as to which users will see which ads based on the advertiser specifications, predictions of ad relevance to individuals, and the ad platform's revenue objectives. 
Targeted advertising has become pervasive and increasingly moderates individuals' exposure to opportunities. In recent years, numerous concerns have been raised about its fairness and discrimination implications, ranging from concerns about discriminatory advertiser targeting practices enabled by the platforms~\cite{FacebookHousingDiscrimination, FacebookJobsOnlyMen, FacebookJobsAge} to concerns about the ad delivery and allocation algorithms run by the platforms introducing bias where none was intended by the advertiser~\cite{datta2015automated, lambrecht2018algorithmic, ali2019discrimination, upturn2018, tobin2019pb}.
As part of a lawsuit settlement, the most prominent targeted advertising platform, Facebook, has begun to take steps to ensure advertisers cannot discriminate in their targeting practices~\cite{sandberg2019}. However, the question of how to ensure that the ad delivery and allocation algorithms do not lead to discrimination is wide open~\cite{ali2019discrimination, CivilRightsAudit}, in part due to lack of agreement over fairness definition(s) and ad platforms' concerns that existing definitions will restrict allocations in ways that significantly impact their revenue.
As such, the multiple-task setting in the presence of individual preferences provides an important model to investigate formal guarantees of non-discrimination without being overly-restrictive for the decision-maker.

In the multiple task setting, 
we think of the set of outcomes $\C$ as arising
from a collection of distinct tasks, e.g.\ deciding whether to show an ad for a user of each ad campaign $c \in \C$.
Importantly, in this setting, a separate fairness metric $d_c$ is specified for each task (ad campaign), which
naturally models real-world concerns in advertising, where different types of ads (e.g.\ housing, employment, product) are subject to different regulations and standards of fairness.

\begin{definition}[Multiple-task IF]
An allocation $\pi:\X \to \Delta(\C)$ is said to be $(D, \left\{ d_{1},\dots d_{k}\right\} )$-\emph{individually fair} in the multiple-task setting if for every two individuals $i,j \in \X\times \X$,  the task-specific Lipschitz condition holds for each task:
$$ \forall c\in \C: \,\,\,\, 
D(\pi(i)_c,\pi(j)_c) \le d_c(i,j).$$
\end{definition}

In this setting, and particularly its application to ad delivery in the targeted advertising context,
the benefits of a preference-informed approach to ensuring fairness become particularly salient.
For instance, consider the following example, due to \cite{DworkI19}.
Suppose there are two ad campaigns, one for a high-paying tech job and another for childrens' toys. The ad-specific metrics capture the fact that differentiating based on a particular criteria could be permissible in some cases and not in others. For example,
the metric associated with the tech ad should assign a small distance to individuals of similar qualifications regardless of their status as a parent, whereas the metric for toys might reasonably assign significant distance between parents and non-parents. However, under Multiple-task IF, a parent that is qualified for the tech job ad but is interested in toys must see the tech ad with the same probability as a qualified non-parent -- an overly restrictive requirement.
PIIF in the multiple-task setting addresses precisely this issue.

\begin{definition}[Multiple-task PIIF]
An allocation  $\pi:\X \to \Delta(\C)$
satisfies $\nobreak (D,\left\{ d_{1},\dots d_{k}\right\} ,\preceq)$-preference-informed individual fairness in the multiple-task setting if 
for all individuals $i \in \X$, for all other individuals
$j \in \X$, there exists an allocation $p^{i;j} \in \Delta(\C)$
such that:
\begin{gather*}
\forall c\in \C: \,\,\,\,  D\left(p^{i;j}_c, \pi(j)_c\right) \le d_c(i,j)\label{def:PIIF:if2}\\
\pi(i) \succeq_i p^{i;j} \label{def:PIIF:pref2}
\end{gather*}
\end{definition}
Again, in the multiple-task setting, the preference-informed extension of IF will require that for every individual $i \in \X$, when comparing to
every other individual $j \in \X$, the individual $i$ prefers their actual allocation to some alternative allocation, $p^{i;j}$.
The main distinction is that now $p^{i;j}$ has to satisfy multiple-task IF with respect to $j$'s current allocation.

\textbf{Efficient optimization.} Our results regarding efficient optimization subject to PIIF from the single-task setting (Section \ref{sec:optimization}) directly extend to the multiple-task setting. In particular, given the ad-specific metrics, individuals' utilities and the advertisers' bids, the platform can efficiently compute the revenue- (or social welfare-) maximizing PIIF allocation.

An interesting direction for future work is relaxing the full information assumption.
In particular, an online model, in which allocations are determined on a per-user basis, could naturally be more applicable, as well as allow for the preferences to be ``discovered'' through the allocation procedure (see \cite{gillen2018online} for a similar approach wrt the metric itself). This may necessitate investigation of non-trivial tradeoffs, as learning individuals' preferences requires some exploration, which may be at odds with ensuring fair treatment.

\subsection{Fairness and social welfare in the multiple-task setting}
\label{sec:ads:gap}
\label{sec:social}
The construction of Proposition~\ref{prop:single:gap} demonstrates that in the single-task setting, the gap between the best social welfare obtainable under IF and PIIF can be large even under very structured classes of preferences.
This construction can be generalized to the multiple-task setting;
however, an unconvincing aspect of it is the requirement that every individual is identical according to the metric.
In such a setting, it's not surprising that IF is overly-constrained.

Here, we describe a family of instances in the multiple-task setting where the per-task similarity is \emph{perfectly aligned with individuals' utilities}; that is, if two individuals benefit similarly from an outcome $c$, then they are similar.
In such instances,
we'd 
expect that the metric constraints would be perfectly aligned with social welfare.
Still, we show that this intuition does not carry through for multiple-task IF: for a set of tasks, there are instances where the optimal social welfare under PIIF approaches a factor $\card{\C}$ larger than the best IF solution.

\begin{theorem}
\label{thm:multiple_task_sw}
For any constant $\eps > 0$, there is a sufficiently large $\card{\X}$ such that there exists a distribution of multiple-task instances where for each task $c \in \C$, $d_c(i,j) = \card{u_i(c) - u_j(c)}$ and $$\frac{\W^*}{\W^*\left(\Pi^{\IF}\right)} \ge {\card{\C}} - \eps.$$
\end{theorem}

\textbf{Intuition.} Our proof is inspired by a construction of \cite{DworkI19}, which shows the impossibility of multiple-task IF under ``naive composition.''
We begin by adapting their construction to our setting.
Suppose there are two subpopulations of individuals $S\subseteq \X$ and $T = \X \setminus S$. We assume that each task-specific similarity metric $d_c$ is determined by individuals' utility: $d_c(i,j) = \card{u_i(c) - u_j(c)}$. 
Additionally, suppose there are two ad campaigns $c_0$ and $c_S$.
$c_0$ is a generic campaign where for all individuals $i \in \X$, $u_i(c_0) = 1$; thus, $d_{c_0}(i,j) = 0$ for all $i,j \in \X \times \X$.
$c_S$ is targeted where subpopulation $S$ receives nontrivial utility, but the rest of the population receives no utility;
thus, $d_{c_S}$ treats pairs within $S \times S$ similarly, pairs from $T \times T$ similarly, but for $i,j \in S \times T$, is arbitrarily large, say $d_{c_S}(i,j) = 1$.

Given these campaigns, a natural allocation of ads to individuals, which we call $\pi^\W$, deterministically assigns $\pi^\W(i) = c_S$ to all individuals in $i \in S$ since they receive positive utility from $c_S$.
Further, it assigns the untargeted campaign $\pi^\W(j) = c_0$ to individuals in $j \in T$ because they benefit positively from seeing $c_0$, whereas they get no benefit from $c_S$.
Indeed, $\pi^\W$ maximizes the social welfare; everyone sees their favorite ad.
But $\pi^\W$ violates multiple-task IF on $c_0$; that is, for $i,j \in S \times T$,
$\card{\pi^\W(j)_{c_0} - \pi^\W(i)_{c_0}} = 1$ but $d_{c_0}(i,j) = 0$.
Intuitively, under multiple-task IF, because everyone in $\X$ is similar according to $c_0$, the platform must decide whether it is more beneficial to show $c_S$ to the individuals in $S$ at the expense of not being able to show $c_0$ to the individuals outside of $T$.
The proposition follows by extending this construction beyond the case of two campaigns and two subgroups, and carefully constructing utility functions for individuals.

\begin{proof}
Suppose $\card{\C} = n$.
Let $t \in \N$ be some constant.
Suppose the universe of individuals $\X = S_0 \cup S_1 \cup
\hdots S_{n-1}$ is partitioned into disjoint subpopulations
($S_\ell \cap S_m = \emptyset$ for $\ell \neq m$) for $\card{\X} \ge t^n$.
The subpopulations will become progressively smaller as
$\ell$ increases; for each $\ell > 0$,
$\frac{\card{S_\ell}}{\card{\X}}= 1/t^\ell$
and let $\frac{\card{S_0}}{\card{\X}} = 1 - \sum_{\ell=1}^{n-1}\frac{\card{S_\ell}}{\card{\X}}$.

Notationally, for all $\ell \in [n]$, let $T_\ell = \bigcup_{m \ge \ell} S_m$.
We construct individuals' utilities as follows:
for each $\ell \in [n]$, for each individual $i \in \X$,
\begin{equation*}
u_i(c_\ell) = \begin{cases} t^\ell & \text{if } i \in T^\ell \\
0 & \text{otherwise}  \end{cases}
\end{equation*}

First, note that for individuals $i \in S_\ell$, $c_\ell$ maximizes their utility $u_i(c_\ell) = t^\ell$.
So consider the welfare-maximizing allocation that assigns every individual in $S_\ell$ to campaign $c_\ell$; this allocation satisfies PIIF.
The average social welfare can then be written as:
\begin{align*}
\frac{1}{\card{\X}} \cdot \sum_{\ell = 0}^{n-1}\sum_{i \in S_\ell} u_i(c_\ell)
&= \frac{\card{S_0}}{\card{\X}} + \sum_{\ell = 1}^{n-1} \frac{\card{S_\ell}}{\card{\X}} \cdot t^\ell\\
&= \left(1 - \sum_{\ell = 1}^{n-1} t^{-\ell} \right) + \sum_{\ell = 1}^{n-1} t^{-\ell} \cdot t^\ell\\
&= n - \sum_{\ell = 1}^{n-1} t^{-\ell}
\end{align*}

Next, consider similarity metrics defined by the utilities:
For each task $c_\ell$, we take $d_{\ell}(i,j) = \card{u_i(c_\ell) - u_j(c_\ell)}$.
By the definition of $u_i$, under these similarity metrics, every pair of individuals $i,j \in T_\ell \times T_\ell$ are considered similar $d_\ell(i,j) = 0$. 
We fix $D$ to be the total variation distance, in other words, $D(p_c,q_c) =\card{ p_c -q_c}$.
As such, any allocation $\pi$ that satisfies IF must show $c_\ell$ to every individual $i \in T_\ell$ with some fixed probability $\alpha_\ell = \pi(i)_{c_\ell}$.
We can compute the expression for the average social welfare of any such assignment as a function of the $\alpha_\ell$.

\begin{align}
\frac{1}{\card{\X}} \cdot \sum_{\ell = 0}^{n-1} \alpha_\ell \cdot \sum_{i \in T_\ell} u_i(c_\ell)
&= \alpha_0 \cdot \frac{\card{S_0}}{\card{\X}} + \sum_{\ell = 1}^{n-1} \alpha_\ell \cdot t^\ell \cdot \frac{\card{T_\ell}}{\card{\X}} \notag \\
&= \alpha_0 \cdot \left(1 - \sum_{\ell = 1}^{n-1} t^{-\ell}\right) + \sum_{\ell = 1}^{n-1} \alpha_\ell \cdot t^\ell \cdot \sum_{m = \ell}^{n-1} t^{-m} \label{multi:eqn:S}\\
&\le \alpha_0 + \sum_{\ell = 1}^{n-1} \alpha_\ell + \sum_{\ell = 1}^{n-1} \alpha_\ell \cdot \sum_{m = \ell+1}^{n-1} t^{-m+\ell} \notag \\
&\le \left(\sum_{\ell = 0}^{n-1} \alpha_\ell\right) \cdot \left(1 + \sum_{m=2}^{n-1} t^{-m}\right) \label{multi:eqn:upper}\\
&= 1  + \sum_{m=2}^{n-1} t^{-m}\label{multi:eqn:final}
\end{align}
where (\ref{multi:eqn:S}) follows by expanding $\frac{\card{T_\ell}}{\card{\X}}$ in terms of $\frac{\card{S_m}}{\card{\X}} = t^{-m}$;
(\ref{multi:eqn:upper}) applies H\"{o}lder's inequality; and (\ref{multi:eqn:final}) uses the fact that individuals $i \in S_{n-1}$ are members $i \in T_\ell$ for all $\ell \in [n]$, so the sum of the probabilities
$\sum_\ell \alpha_\ell \le 1$.

Given a desired $\eps$, we can take $t$ large enough, the ratio between the social welfares exceeds $n - \eps$.
\end{proof}

\paragraph{Optimality under IF.}
Intuitively, this construction highlights the fact that allowing further targeting and more ad campaigns to participate allows the gap in social welfare between the best IF and PIIF solutions to grow considerably.
Note that this gap applies even if
the platform's objective is to optimize social welfare, so the proof also shows a gap in worst-case utility achievable by the decision-maker under IF.

We remark that a corollary of our result
is that the Dwork-Ilvento ``RandomizeThenClassify'' mechanism \cite{DworkI19}
for composition under multi-task IF achieves
worst-case optimal performance
(in terms of both social welfare and utility to the platform).
In particular, \cite{DworkI19} give an algorithm (in a setting with limited information modeling ``competitive composition'') that allocates a fixed distribution $p \in \Delta(\C)$ to all individuals -- thus, satisfying IF -- that achieves a $1/\card{\C}$-fraction of the best unconstrained utility.
Our result shows that no IF solution, even with full information, can achieve a better fraction of the achievable utility.

\vspace*{-4mm}
\section{Discussion}
\label{sec:discuss}

In this section, we review additional related work, note some possible extensions within the preference-informed fairness framework, and conclude with a discussion of the strengths and limitations of our current approach.

\subsection{Further related works}

Since \cite{dwork2012fairness}, a number of recent works have aimed to extend the ``fairness through awareness'' framework, including \cite{yona2018probably,kim2018fairness,gillen2018online,kearns2019average,jung2019eliciting}.
These works focus on translating the theoretical IF framework into practically-motivated settings.

Three recent works have suggested incorporating notions of individuals' \textit{preferences} into the fairness definitions.
First, \cite{balcan2018envy} present EF as an alternative to IF
and study its learning-theoretic properties.
Their focus is on the question of generalization: given a classifier that is envy-free on a sample, is it approximately envy-free on the underlying distribution?
Their main technical result is a positive answer to this question, when learning from a particular structured family of classifiers.
An interesting open question is whether the generalization results for IF from \cite{yona2018probably} and EF from \cite{balcan2018envy} can be combined to give generalization for PIIF.

Second, \cite{zafar2017parity} considers two notions of fairness at the (weaker) group level: treatment parity and impact parity. Their main contribution is a relaxation of both definitions, allowing for solutions where every protected group is ``better off'' \emph{on average}.
From a technical perspective, achieving their notion requires solving a non-convex optimization problem even in the simple case of linear classifiers for two disjoint groups.  
Our approach is different in that it focuses on defining both fairness and preferences at the \textit{individual} level. This allows for a significantly stronger fairness  guarantee, as well as a much more general framework that supports any notion of benefit or preference individuals may have. Importantly, our notion provably admits efficient optimization for a rich class of preference relations.

Finally, independent of our work, \cite{chawla2019individual} study and quantify trade-offs between individual fairness and utility in an online version of the targeted advertising problem.
\cite{chawla2019individual} also observe that IF can come at a high cost to utility in the multiple-task setting of \cite{DworkI19}, but 
propose a different relaxation.
Under their notion, every individual $i$ chooses a subset of the outcomes $S_i \subseteq \C$ and is guaranteed that their probability of seeing an ad from $S_i$ is greater than the probability of every other individual of seeing an ad from $S_i$.  Such a guarantee can be viewed as a variant of EF over a more restricted class of preferences than those we consider.
The main distinction from PIIF is that this notion ignores the distance metrics entirely and in this sense resembles envy-freeness more than individual-fairness.

\subsection{Preference-informed group fairness}

\label{sec:groups}
In this work, our focus was on incorporating individual preferences into the metric-based individual fairness framework Dwork~\emph{et al.} \cite{dwork2012fairness}. The space of fairness definitions, however, is large, and different definitions may be more appropriate in different contexts.

A different approach for defining fairness, often referred to as ``group fairness,'' proceeds as follows. A protected attribute, such as race or gender, induces
a partition of the individuals into a small number of groups. For simplicity, we focus on the case where there is a single protected group, $S$, where the rest of the population is denoted $T = \X \setminus S$.
A classifier is considered fair
if it achieves parity of some statistical measure across these groups. Group fairness notions are typically weaker than individual notions of fairness: they only provide a guarantee for the ``average'' member of the  protected groups and might allow blatant unfairness towards a single individual or even large subgroups;
indeed, the shortcomings of group notions  motivated the original work on ``fairness through awareness'' and subsequent works \cite{dwork2012fairness, kearns2017preventing, hebert2017calibration, kim2018fairness}. Although group fairness notions can be fragile, they are widely studied and used due to their simplicity and due to the fact that they are easier to enforce and implement (for example, they do not require a task-specific similarity metric).

In principle, much of the reasoning behind our argument for incorporating preferences into IF \cite{dwork2012fairness} also extends to group-fairness notions. In this section, we show how we might augment a common group notion, called \emph{Statistical Parity} (SP), to incorporate preferences.
When there is a clearly ``desirable'' outcome, SP aims to protect the group $S$ by guaranteeing equal average exposure to the desired outcome.
\begin{definition}[Statistical parity]
\label{def:sp}
A binary classifier $h:\X\rightarrow\left\{ \pm1\right\}$  satisfies (exact) statistical parity with respect to $S$ if 
\begin{equation*}
    \Pr_{i\sim S}\left[h(i)=1\right]=\Pr_{i\sim T}\left[h(i)=1\right]
\end{equation*}
\end{definition}

In our context, when individuals have diverse preferences over the outcome space, enforcing SP may again come at a cost to members of $S$, the group that SP aims to protect.
As a concrete example, suppose everyone in $\X$ prefers the outcome $+1$, with the exception of some fraction of $S$, denoted $S'$,  who prefer the outcome $-1$. In this case, the statistical parity constraints prevents the solution 
where $h(i) = +1$ for $i \in X \setminus S'$ and $h(i) = -1$ for $i \in S$,
which from the individuals' perspective is optimal. 

Building on this intuition,
we extend the set of classifiers we deem fair.
Assuming every individual $i\in X$ has a preference relation over  $\left\{ \pm1\right\}$ (or even distributions over $\left\{ \pm1\right\}$), \emph{preference-informed statistical parity} (PISP) allows deviations from SP, as long as they are aligned with the individuals' preferences.

\begin{definition}[Preference-informed statistical parity]
\label{def:pasp}
A binary classifier $h:X\rightarrow\left\{ \pm1\right\}$  satisfies preference-informed statistical parity with respect to $S$ if there exists an alternative classifier, $h':X\rightarrow\left\{ \pm1\right\}$ , such that:
\begin{gather*}
\label{eqn:pasp1}
    \forall j\in T,	h'(j)=h(j) \\ 
\label{eqn:pasp2}
\forall i\in S,	h(i)\succeq_{i}h'(i) \\
\label{eqn:pasp3}
	\Pr_{i\sim S}\left[h'(i)=1\right]=\Pr_{i\sim T}\left[h'(i)=1\right] 
\end{gather*}
\end{definition}

That is, fixing the outcomes members of $T$ receive under $h$, every single member of $S$ prefers their current outcome over what they would have received under a classifier satisfying statistical parity. Importantly, the guarantee is still with respect to the preferences of the \textit{individual} members of $S$. 

We conclude with several remarks regarding PISP.
First, note that PISP only enriches the set of solutions that satisfy SP; any classifier that satisfies SP also satisfies PISP, by taking the alternative $h' = h$.
The classifier welfare-maximizing classifier, where each individual is assigned their favorite outcome, is considered fair.
For example, revisiting our example above, the classifier that gives $+1$ to $\X \setminus S'$, and $-1$ to $S'$ is fair, because the alternative classifier that gives \textit{everyone} $+1$ satisfies the PISP constraints.
Finally, we argue that PISP maintains the core of the fairness guarantee of SP.
For example, consider the classifier that assigns $+1$ to members of $T$ and $-1$ to members of $S$.
This classifier benefits the members of $T$ in a way that is \textit{not} aligned with the preferences of all members in $S$; rightfully, it does not satisfy PISP, because $i \in S \setminus S'$ are harmed.

\subsection{Revisiting the assumptions underlying PIIF}
Three main assumptions underlie our work.
The first is that the outcome space is taken as a given.
This could be problematic if the outcomes themselves are biased, e.g.,  tailored to the preferences of the majority, or worse yet, harmful to the minority. 
A biased outcome space would also present a problem for both IF and EF, which PIIF does not escape entirely.
Still, PIIF may ameliorate the issue, by allowing the minority to receive outcomes that they prefer. We see a study of fairness of the outcomes themselves as an exciting direction for further inquiry.

The second assumption is that any deviation from an IF solution that is aligned with individuals' preferences should still be considered fair.
This assumption follow from the perspective that ``fair'' allocation algorithms should protect the welfare of individuals; this perspective naturally extends the perspective underlying IF that similar individuals should be treated similarly.
As discussed in  \cite{hellmanTwo,hellmanMeasuring},
in other (legal) settings, the notion of ``fairness'' may necessarily imply ``treatment as equal,'' and notions of individual fairness may not apply.
In such settings, the societal notion of fairness may require going against individual preferences.
Handling such settings lies beyond the scope of our current work that focuses on the computer science notions of individual fairness, in which the objective is to provide strong protections from discrimination to the individuals themselves.

The third assumption is that the individuals' preferences are known.
This is certainly the most nontrivial technical assumption we make; nevertheless, there are established techniques for learning utility-functions from observed behaviour \cite{chajewska2001learning}.
We also note that accurately learning preferences could often be in the interest of the decision-maker (for example, ad platforms often claim that their implementations of targeted advertising are in line with users' interests~\cite{zuck2019wsj}). Still, any practical estimation of the preferences runs the risk of injecting further bias during the learning process (for example, if the minority's preferences are estimated with lesser accuracy) and, therefore, mandates special attention in future research.

\paragraph{Acknowledgments.}  \emph{This work grew out of conversations during the semester on Societal Concerns in Algorithms and Data Analysis (SCADA) hosted at the Weizmann Institute of Science.  The authors thank Omer Reingold for helpful conversations, which influenced our understanding and the presentation of the work.}

\balance
\bibliographystyle{alpha}
\bibliography{refs}

\end{document}